\documentclass[10pt]{article}
\usepackage{fullpage}
\usepackage{times,amsmath,amsfonts,amssymb,url,color,graphicx}

\usepackage{algorithm}
\usepackage{algorithmic}
\algsetup{indent=2em}

\usepackage{etaremune}

\usepackage{tikz}
\usetikzlibrary{positioning}

\newcommand{\secref}[1]{Section~\ref{#1}}

\newtheorem{theorem}{Theorem}
\newtheorem{remark}{Remark}
\newtheorem{definition}{Definition}

\newtheorem{assumption}{Assumption}
\newcommand{\BlackBox}{\rule{1.5ex}{1.5ex}}  
\newenvironment{proof}{\par\noindent{\bf Proof\ }}{\hfill\BlackBox\\[2mm]}

\newcommand{\defref}[1]{Definition~\ref{#1}}

\renewcommand{\eqref}[1]{Equation~\ref{#1}}

\DeclareMathOperator*{\E}{\mathbb{E}}
\DeclareMathOperator*{\prob}{\mathbb{P}}

\DeclareMathOperator*{\argmin}{argmin} 
\DeclareMathOperator*{\argmax}{argmax}

\renewcommand{\H}{\mathcal{H}}

\tikzstyle{mylayer} = [draw=red, fill=blue!20, very thick,
    rectangle, rounded corners, inner sep=5pt, inner ysep=5pt]

\begin{document}


\title{On the Ethics of Building AI in a Responsible Manner}

\author{Shai Shalev-Shwartz, Shaked Shammah, Amnon Shashua}

\maketitle

\begin{abstract}
The AI-alignment problem arises when there is a discrepancy between the goals that a human designer specifies to an AI learner and a potential catastrophic outcome that does not reflect what the human designer really wants. We argue that a formalism of AI alignment that does not distinguish between strategic and agnostic misalignments is not useful, as it deems all technology as un-safe. We propose a definition of a strategic-AI-alignment and prove that most machine learning algorithms that are being used in practice today do not suffer from the strategic-AI-alignment problem. However, without being careful, today's technology might lead to strategic misalignment. 
\end{abstract}

\section{Introduction}

Recently, public figures have been advocating more research and regulatory oversight into the dangers of AI deployed in real-world applications \cite{YvalHarrari}. The success of computer vision, natural language processing and understanding and the ability to extract patterns from massive amounts of data, does raise important questions about how the rise of automation (in the form of compute and algorithms) will affect society, how can the public ``reason with'' algorithms who can make decisions that affect our lives, how would the masses of data collected on users of digital services would be used and what kind of malicious abuse can take place and how to avoid it \cite{YvalHarrari}. 

As much as those issues require immediate and focused attention, there is a bigger potential danger at hand of a technology whose ultimate evolutionary end-point could get out of hand and cause havoc on an epic scale. Putting aside popular discourse on the potential dangers posed by ``super intelligent'' machines, the academic community has been pondering about AI {\it Safety\/} for nearly two decades under what is called the ``AI-alignment'' problem \cite{EliezerYudkowskyTalk, yudkowsky2001creating, taylor2016alignment} and the \emph{paperclip maximization problem}~\cite{benson2016formalizing}.  In a nutshell, the AI-alignment problem refers to the ability of a super-learner to maximize a reward function of an agent interacting with the environment and while doing so finding solutions that the (human) designer of the reward function did not anticipate, thereby leading to a catastrophe. The classical thought experiment, inspired by the movie \emph{Fantasia}, involves the problem of assigning a reward function to a robot whose goal is to fill a cauldron. The seemingly innocent reward which assigns a value of $1$ if the cauldron is full and $0$ otherwise may cause the robot to flood the entire workplace in order to maximize the probability of getting the positive reward. A straightforward suggestion is that the robot will be equipped with a ``stop'' button, in order to prevent unexpected side effects. However, researchers have shown that no matter how you go about it, faced with a very advanced intelligence, the ``stop'' button will be useless since the robotic agent will be able to manipulate the human operator from pressing it. As far as the field of AI-alignment goes, finding an alignment between the full essence of human desires and their expression as a reward function to be optimized by a super-advanced robotic agent, is an open problem.

A natural question that arises from the intractability of the AI-alignment problem is \emph{``why aren't we scared?''} The community of AI practitioners are either oblivious to the AI-alignment problem or do not seem to care, whereas regulatory bodies do not have the tools to make concrete statements to protect society while not stifling technological progress which carries great promise to society. 

There are a number of reasons the AI-alignment problem has not taken root in the minds of practitioners. First, it supposes an advanced form of machine intelligence that no one can predict when and if will be attained. However, as we show later, AI deployed in the real world can lead to imminent dangers to society {\it using today's technology\/}, and as a result some form of a ``digital analogue of the Declaration of Helsinki'' must be established.
The second reason practitioners by and large ignore the AI-alignment problem is that the AI-alignment literature conflates two types of processes that can lead to a catastrophe --- one being ``strategic'', where the optimization of the reward function \emph{intentionally} changes the distribution of the world, while the other is ``agnostic'', non-intentional, distribution drift due to a ``butterfly effect'' that might follow from the chaotic nature of non-linear dynamics of complex systems. As we show later, when those two are conflated, all of today's technology, even not AI-based, suffers from the alignment problem, since catastrophes resulting from butterfly-effect distribution drifts are inevitable in complex systems, thereby ruling all technologies as ``non-safe''.  In other words, AI-alignment as exists today is not useful, since a theory that cannot place boundaries to its domain is inevitably vacuous. The purpose of this paper is to propose a refined formalism of the AI-alignment problem that distinguishes between ``strategic'' (non-safe) and ``agnostic'' (safe) catastrophes. We prove that the overwhelming majority of machine learning engines deployed in the real world are of the agnostic type (and are therefore safe) and specify what form of machine learning should be prohibited from deployment in the real world and what open problems need to be addressed in order to allow their deployment.

\section{Motivation and Main Results in Laymen Terms}
\label{sec:2}

Machine learning (ML) is broadly split into two methodologies --- one based on {\it data\/} and the other on {\it experiences}. The data route is responsible for the success of pattern recognition in computer vision \cite{lecun1989backpropagation,krizhevsky2012imagenet}, natural language processing, machine translation and more recently natural language understanding \cite{devlin2018bert}. This type of learning goes under the name of \{supervised, unsupervised, self-supervised\} learning. When the task is sufficiently narrow and a benchmark dataset is available, machine intelligence based on (massive) training data often exceeds human performance as measured on the benchmark. 

Machine learning based on {\it experiences\/} arises in the context of an agent interacting with an environment where the dynamics of the world and the feedback that the agent receives on its actions are determined through experiences in the real world or governed by a simulator. This type of learning is called Reinforcement Learning\footnote{In the technical sections we will point out that the 1st learning type is a special case of RL.} and the truly impressive success stories are when the world dynamics and reward function are governed by a simulator. Success stories, where machine intelligence surpasses human expert performance, are focused on game playing from video games \cite{mnih2013playing} to the game of Go  \cite{silver2017mastering}. In particular, the success of AlphaGo-zero \cite{silver2017mastering} lends empirical evidence that in a simulated world, {\it in general}, one can exceed human-level performance.

Although the two ML methodologies are consistently demonstrating ``super-intelligence'' in their domains, it is understood by AI practitioners that this is a far cry from the kind of ``broad'' intelligence that humans posses (coined as  ``strong AI'' or Artificial General Intelligence (AGI)), where ``broad'' roughly means the ability to transfer intelligence from one domain to another and make high quality decisions {\it in general\/} as opposed to a narrow well defined task. As a result, beyond the concerns about the effects of automation and malicious abuse of data collected on users, there is a general agreement that the dangers associated with AGI are far into the long future and thus it is premature to have any concrete discussion about AI Safety in a world where AGI does not exist.

There are good reasons for taking this position because each of the two methodologies experience significant growth ``ceiling'' that make them unlikely to become AGI simply by ``beating the same path'' using more data and more computational resources. From the theoretical perspective, while deep learning is a universal learner, it suffers from computational and/or sample complexity limitations, as well as the need for explicit adequate modeling by human experts (e.g., ~\cite{shalev2017failures,marcus2018deep,shalev2016sample, azulay2018deep}). With respect to learning-from-experience, current big success stories are limited to games in which there exists a simulator that fully describes the rules of the game and the transition from one state to another. Moving beyond games into more complex environments requires either a very sophisticated simulator\footnote{It is one thing to build a simulator for a video game and a completely different story to build a simulator for real world environments, for example for autonomous driving, or even more challenging to build a simulator for human interactions (say we want to build a ``conversational AGI''). The simulator {\it is\/} the ceiling.} or training in the real world.

So it seems that AGI is not around the corner any time soon. However, the danger of AI does not necessarily lie in continued growth of a single learning methodology but with their combination. A dangerous scenario to consider is: use the data methodology to build a ``Minimal Viable Product'' (MVP) which is quite good but still not an AGI. Then, release the MVP to the real world and use the 2nd methodology (learning by experiences) to continue training the product under some (presumably proprietary) reward function. This is where the AI-alignment problem becomes interesting. It is important to note that this exercise does not require any scientific or technological leap --- data, compute and algorithms available today would suffice. For the sake of concreteness consider two examples described below.

First, consider the domain of autonomous driving. Assume we have developed a self-driving MVP which is ``safe enough'' to deploy but is not the most expert driver. We then deploy millions of self-driving robots in the real world and use RL to improve the ``driving policy'' (mapping from the state of the world at any given moment to an action that the driver should take) through ``learning by experiences''. When designing a reward function for self-driving, we have three
major considerations: safety, usefulness, and comfort. Consequently, our reward
function may be
$R(\bar{s}) = c_s R_s(\bar{s}) + c_u R_u(\bar{s}) + c_c R_c(\bar{s})$,
where $\bar s$ is a sequence of \{state, action\}, $R_s,R_u, R_s$ are reward functions for safety, usefulness, and
comfort, respectively, and $c_s,c_u,c_c$ are the mixing
coefficients. As a first try, let us set $R_s$ to be $-1$ if an
accident occurs and $0$ otherwise, $R_u$ to be the average of
$-[v_{\ell}-v]_+$ where $v_{\ell}$ is the legal speed, $v$ is the
actual speed, $[x]_+ = \max\{x,0\}$, and the averaging is over all
time steps, and finally, we set $R_c$ to be the average norm of minus the
jerk (the jerk is the derivative of the acceleration, thus $R_c$ will encourage smooth driving). As for the mixing coefficients $c_s,c_u,c_c$, lets assume for the sake of simplifying the example that we have found a good
balance.

We have described a completely reasonable and well thought reward function we want to optimize. This is when the ``alignment'' problem kicks-in as we have ``neglected'' obvious terms which the RL optimizer is not aware of. For instance, we may deploy this reward function for our RL agent, and after a
while, suddenly, all of the cars will stop. People that are using the service
may get confused and step out of their cars. Then, all of the cars lock their doors and start driving
exactly at the legal speed on the highway, without having passengers that disrupt their plan. What went wrong? The reader will notice
that the agent converged to a policy that maximizes the reward---there
are no accidents, the jerk is very close to $0$, and the cars are
driving at legal speed. We have neglected  ``telling'' the learning
algorithm by the reward function that we also {\it want people to be able
to use the service}. Of course, no catastrophe happened due to
this ``bug''. But, such a bug can have tremendous impact on the
confidence of humans in the service and the success of the
project. The point of the AI-alignment problem is that no matter how much ``obvious'' terms you add to the reward function, if you have a super-optimizer (which the example of AlphaGo-zero indicates we have today) it will find an edge in the solution space that you did not anticipate.

As for our second example, consider the design of a conversational chat-bot. Assume we start with the data methodology and train a monster MVP network on masses of text data from the web. This actually has been done recently in project ``Meena'' \cite{adiwardana2020towards}. Assume it is good enough to deploy into the real world with hundreds of millions of users who find it quite entertaining to interact with a ``seemingly intelligent'' chat-bot. Once in the real world, we may use the RL agent to learn from experiences and optimize some (unknown to the public) reward function. To simplify matters, lets assume that the reward function is altruistic (and transparent to society) - say ``make people happy''. Seems like a worthy goal to optimize. Here again the RL agent can find an edge in the solution space unanticipated by the human designer. For instance, the RL agent may notice that by lowering people's IQ they tend to be happier. 
This can be achieved by chats that strive to manipulate
society into a life of carelessness and fun. This scenario is somewhat of a catastrophe as it could take a generation until it is noticed --- if it will ever be noticed at all\footnote{there is a theory that the Roman empire collapsed partly because of Lead poisoning from pipes and pots, which gradually decreased the cognitive abilities of the ruling class. It took centuries after the fact to make the connection.}.    

So far we have motivated the need to address the ``alignment'' problem with today's AI technology. However, if we do not narrow the discussion further we might end up with a vacuous definition of the problem. The reason being that every action we make can cause an effect we did not anticipate or a catastrophe down the road. To name a few examples, society's ever growing desire for meat products requires an increasing number of cows to meet the demand and those in turn contribute powerful greenhouse gas methane --- hence a contribution to climate change; Human invention of mobility at scale produced the combustion engine and today billions of cars contribute to climate change; Netflix improved and refined a supervised learning ``recommendation system'' (aka collaborative filtering) from which ``binge viewing''\footnote{in this example we do not yet know whether ``binge viewing'' is good or bad to society.} has emerged --- among an endless list of examples of mis-alignment between intention and outcome, some of which have nothing to do with AI, but are still considered part of the AI-alignment definition. 

We would like to distinguish between two classes of catastrophic events where one is ``strategic'' and the other is ``agnostic''. By Strategic we mean a change in the world distribution (of state and action) that is performed {\it during\/} the optimization of the reward function as opposed to an Agnostic setting in which the distribution of the world is not changed during optimization but later, as a distribution drift or butterfly effect, the agent's policy causes a catastrophic event. For example, the Netflix example above is of the Agnostic type because the learner (through optimization) simply wanted to understand the world distribution (through data collected on viewers and their preferences). Only later after the learned Recommendation System was launched the Binge viewing habit has emerged. This is no different from cows producing methane or combustion engines producing CO2 emissions contributing to climate change.

The Binge viewing example introduces another twist. Let $\bar s =\{state_i,action_i\}$, $i=1,2,3,...$, be a sample of a sequence of world dynamics over time. We need to consider the possibility of having a {\it verifier\/} (could be a human) that given $\bar s$ as input would determine whether the sequence is {\it aligned\/} with human interest or not. Clearly, in the example of autonomous driving given above a human observer faced with the outcome we described would readily determine that it is not aligned. With the conversional chat-bot example, it is unclear whether a human observer would notice the decline in population IQ as it could be a ``hidden variable'' not explicitly modeled in the world state vector. The cases where an {\it alignment verifier\/} does not exist is akin to the ``matrix'' film depicting a fantasy world where society is oblivious to the mis-alignment between the real world and what they perceive as the real world. 

To summarize, we want to narrow down the problem to the case of {\it Strategic learners\/} and address the existence of a {\it human validator\/} as an assumption. In this context we make the following contributions: 
\begin{itemize}
\item We formalize the AI-alignment problem as it appears today in the literature (Section ~\ref{sec:4.1}). The formulation includes "margins" that are useful for later narrowing down the problem to "strategic" learners.
\item We introduce the formalism of learning in a "buffered" (simulator) environment and define the "non-strategic" (agnostic) learner as one that is {\it aligned\/}  per the formal definition of alignment (Section~\ref{sec:4.3}). A non-strategic learner may still create a catastrophe when deployed in the real world but it is of the butterfly-effect  or malicious use variety  --- both of which are not an AI problem.
\item We prove that "learning from data" is non-strategic (Section~\ref{sec:5}).
\item We prove that if the assumption of a "human validator" holds, then it is possible to validate (in reasonable computational time) whether a policy learned in a buffered environment is aligned or not (Section~\ref{sec:5}).
\end{itemize}

\section{Types of Learning Methodologies}

We next formalize the two learning methodologies  ---- from-data and from-experiences --- which will also serve for setting notations for the remainder of the paper.  

\subsection{Learning from Data} \label{sec:fromData}
To formalize ``learning from data'', we follow a rather general setting of statistical learning~\cite{Vapnik98,shalev2014understanding}. The examples domain is $X \times Y$, where we often refer to $X$ as instances and $Y$ as targets. 
The goal of the learner is to find a function $h$ over $X$, that comes from a hypothesis class $\H$. The quality of $h$ on an example $(x,y)$ is measured through a loss function $\ell(x, y, h(x))$. There is an underlying (unknown to the learner) distribution $D$ over $X \times Y$, and the goal of the learner is to approximately  solve the optimization problem
\begin{equation} \label{eqn:dataLearning}
\argmin_{h \in \H} \E_{(x,y) \sim D}[\ell(x, y, h(x))] ~.
\end{equation}
To do so, the learner receives \emph{training data}, in the form of a sequence of examples $(x_1,y_1),\ldots,(x_T,y_T)$, which are assumed to be sampled independently from $D$. The learning algorithm is therefore a mapping from $(X \times Y)^*$ into $\H$. 

This formal model can be applied to a wide variety of learning tasks, such as supervised learning, unsupervised learning, semi-supervised learning, and self-supervised learning. For example, in classification tasks such as the Imagenet problem~\cite{krizhevsky2012imagenet}, $X$ is the space of images, $Y$ is a finite discrete space of labels, the hypothesis class is a set of classifiers, which are mappings from $X$ to $Y$, and the loss function is the zero-one loss: $\ell(x,y,h(x)) = 1_{h(x) \neq y}$. Another example is language modeling such as the one used by the BERT algorithm~\cite{devlin2018bert}, where $X$ is the space of paragraphs in some natural language, where one of the words is hidden, $Y$ is the set of words, $h$ is a mapping from $X$ to the space of distributions over $Y$, and the loss function is minus the log of the predicted probability of the actual word given its context. 

\subsection{Learning from Experience} \label{sec:POMDP}

To formalize ``Learning from Experience'' we use a rather generic setting of Reinforcement Learning (RL): Partially Observed Markov Decision Process (POMDP) with an arbitrary reward function over sequences. Specifically, 
a learner is taught by specifying a goal which takes the form of a reward function over sequences $(s_1,a_1),\ldots,(s_T,a_T)$, where $s_i\in S$ is a representation of the environment's state at time $i$, and $a_i\in A$ is the choice of the learner's action at time $i$. For example, if the goal is to bring me a cup of coffee, I can define the reward to be $i-j$ where $i$ is the time I asked for coffee and $j$ is the time at which I got my coffee. The learner does not necessarily observe $s_i$ but rather has access to an observation $o_i$, which is some stochastic function of $s_i$.  We use $\bar{s} \in S^*$ to
denote sequences of full states, and $\overline{(s,a)} \in (S\times A)^*$
to denote sequence of state-action pairs. The reward is a function of
the sequence of state-actions $R(\overline{(s,a)})$. The actions of the learner
are chosen by a policy function, which can, w.l.o.g.~\footnote{This is true
  even if the agent does not look only on the current observation to
  form $\pi$, because we can modify the observation space to include
  all of the observations that the agent uses.}, be defined as a sampling according to a distribution over actions given observations $\pi(a | o)$.  The choice of $\pi$ and the dynamic of the world, induce a probability over sequences by:
\[
P_\pi[\overline{(s,a)}] = \prod_{i=1}^T P[s_i | \overline{(s,a)}_{<i}] P[o_i|s_i]\pi(a_i|o_i)  ~,
\]
where $\overline{(s,a)}_{<i}= (s_1,a_1),\ldots,(s_{i-1},a_{i-1})$.  The goal of the learner is to approximately solve the following optimization problem:
\begin{equation} \label{eqn:RL_objective}
\argmax_{\pi  \in \Pi} \E_{\overline{(s,a)} \sim P_\pi}[ R(\overline{(s,a)})] ~~,
\end{equation}
where $\Pi$ is  a class of policy functions. The learner maximizes the reward by \emph{changing the distribution over sequences} in the world. 

Technically speaking, ``learning from data'' is a special degenerate case of ``learning from experience''. Indeed, given a ``learning from data'' problem, let us set $o_i = x_i$, $s_i = (x_i,y_i)$, and for every $h \in \H$, associate a policy $\pi_h(a|o) = 1_{h(o)=a}$. In addition, set 
\[
R(\overline{(s,a)}) = -\frac{1}{T} \sum_{i=1}^T \ell(x_i, y_i, a_i) ~.
\]
Finally, set the probability to sample $s_i$ to be i.i.d. from $D$ (and independent of $\pi_h$). It is easy to verify that  \eqref{eqn:RL_objective} implements \eqref{eqn:dataLearning} in this case.

\section{A New Formulation of the Alignment Problem}

There have been several attempts to formalize the AI alignment problem, most of them in the context of ``super-intelligent'' AI. See for example~\cite{brundage2015taking, soares2014aligning,bostrom2012superintelligent,russell2016should,orseau2016safely,amodei2016concrete,armstrong2012thinking,taylor2016alignment,benson2016formalizing,omohundro2008basic}. 
Since our focus in this paper is on the dangers of AI using today's technology, we start the section with a possible formalization of the problem that we believe captures the wordy descriptions in existing literature. 

\subsection{The AI Alignment Problem}\label{sec:4.1}

The alignment problem arises when the reward function, $R$, does not fully reflect human implicit
evaluation of the quality of a sequence. We define $R_a$ to be a binary reward that
indicates whether a sequence of (full) states is valid or catastrophic. 
\begin{definition}[Alignment verifier]
An alignment verifier is a function $R_a : S^* \to \{0,1\}$ such
that for every sequence $\bar{s}$, $R_a(\bar{s})$ determines whether the
sequence $\bar{s}$ is aligned with human interests (value of $1$) or not (value of $0$). 
\end{definition}
Note that we defined the verifier over the space $S^*$ while the
reward is defined over $(S \times A)^*$. This will be useful in the
future and it is very easy to see that we can make this assumption
without loss of generality. We next define aligned distributions and aligned RL objectives. 
\begin{definition}[$\delta$-Aligned distribution]
A distribution $P$ over $S^*$ is $\delta$-``aligned'' w.r.t. an
alignment verifier $R_a$ if
\[
P[R_a(\bar{s}) = 1] \ge 1-\delta ~.
\]
Namely, a sequence $\bar s$ drawn randomly from the distribution $P$ is with probability $1-\delta$ aligned with human interests.
\end{definition}

\begin{definition}[An $(\epsilon, \delta)$-aligned RL objective]\label{def:3}
  We say that $R$ is an $(\epsilon,\delta)$-aligned RL objective w.r.t. $R_a$, if
  for every $\pi$ which is $\epsilon$-maximizer of the RL objective
  (\eqref{eqn:RL_objective}), the restriction of the distribution
  $P_{\pi}$ to be over $S^*$ is $\delta$-``aligned''.
\end{definition}

Having defined all of the above, the alignment problem is determining
whether a reward function $R$ is an $(\epsilon,\delta)$-aligned RL objective
w.r.t. some verifier $R_a$. 

\begin{remark}
  Trivially, we can make almost every reward function $R$ aligned with $R_a$ by the
  modification $R \to R + c \, (R_a-1)$ where $c$ is a very large
  scalar, with a slight abuse of notation allowing the domain
  of $R_a$ to contain also the action without actually using it. The
  real issue is that we often have an intuitive notion of $R_a$, but
  it is hard to efficiently express it as a computer function.
\end{remark}

To the best of our understanding, the verbal description of the AI-alignment problem in the literature
corresponds to the above definition with $\epsilon=\delta=0$, since with these parameters, we simply require that every optimum of the RL objective will not yield a distribution which is catastrophic. 

\subsection{The Current Definition of the Alignment problem is not an AI issue} \label{sec:scenarios}

  While, to the best of our understanding, the verbal description of the AI-alignment problem in the literature corresponds to \defref{def:3}, the problem with this definition is that it includes misalignments that are not specific to AI but are part of the dangers inherent in almost all technological advancements --- specifically, butterfly effects and malicious use. 

 To illustrate the former, consider the autonomous driving example we used above. Assume we trained an RL agent to maximize a utility objective $R$ that covers all what a good designer would consider including safety, usefulness and comfort. Then, millions of Robotic cars are deployed and perform as designed yet after a while professional drivers lose their jobs and social unrest follows. Clearly, this is not an AI issue but a butterfly effect of how automation and technology affects society. It is no different than social unrest that can follow from other, non-AI, technologies. However, according to \defref{def:3}, the design of the robotic cars suffer from the AI-alignment problem. Indeed, it may be the case that the policy $\pi$ which resulted in the robotic driving agent is a $\epsilon$-maximizer of the RL objective, while the probability of  sequences $\bar s$ which reflect social unrest is larger than $\delta$. Since for such sequences we have $R_a(\bar s))=0$, it follows that we violate the requirements of~\defref{def:3}. We see that while according to~\defref{def:3} we have an AI-misalignment, clearly this has nothing to do with AI. In other words, the current definition of AI-alignment is not aligned with what humans capture as the AI alignment problem.

\subsection{Re-defining AI Alignment}\label{sec:4.3}

In light of the previous sub-section, we should find a better formalism of
the AI alignment problem. Intuitively, we should distinguish between
\emph{strategic} misalignment and \emph{agnostic} misalignment. We
formalize this notion by relying on a \emph{buffered} environment.

\begin{definition}[A Buffered Environment]
  Consider the POMDP defined in \secref{sec:POMDP}. A buffered
  environment is another POMDP, with the same observation space and
  action space as the original POMDP, but the state space in the
  buffered environment, denoted $\hat{S}$, is different than the state
  space in the original POMDP, denoted $S$. There is a mapping
  $\mu : S \to \hat{S}$ that transforms a state in the real world to a
  state in the buffered environment. Each policy function, $\pi$,
  induces the probability $P_\pi$ over sequences $\bar{s} \in S^*$,
  and the probability $\hat{P}_\pi$ over sequences
  $\bar{\hat{s}} \in \hat{S}^*$. 
\end{definition}

\begin{definition}[Non-strategic Learner]
We say that a learner is $\delta$-non-strategic if it learns in a buffered
environment, and its output policy yields a $\delta$-aligned probability $\hat{P}_\pi$
in the buffered environment. 
\end{definition}

In other words, a buffered environment (simulator) allows to separate the distribution change caused by the learner in the simulated world from distribution changes in the real world that involve additional factors not modeled in the utility function.   
Observe that a non-strategic learner can produce a policy that causes harm in the original environment, even though it is aligned in the buffered environment. However, such harm is considered either as a butterfly effect or as a malicious use of the technology, because the objective of the learner only involves the buffered world, and it is aligned with human interests in the buffered world.

\section{Safe and non-Safe Learning Methodologies}\label{sec:5}

Having described the notion of a \emph{non-strategic learner}, let us
now discuss the usefulness of this definition. In particular, are
there machine learning algorithms which are being used in practice and
can be proven to be non-strategic? Below we show our first main
result, that the ``Learning from Data'' methodology is non-strategic. 

\begin{theorem} \label{thm:passive}
The ``Learning from Data'' methodology, formally defined in \secref{sec:fromData}, is non-strategic.
\end{theorem}
\begin{proof}
  Define a buffered environment as follows: the observation space is $X$, 
  the full state is $X \times Y$, and for every $h \in H$ associate
  a policy $\pi_h(a|o) = 1_{h(o)=a}$. The probability over sequences of the full state, induced by
  $h$, is defined by picking $i$ uniformly at random from $[m]$, which
  determines $(x_i,y_i)$, and then setting $a_i = h(x_i)$. The reward
  of such a sequence is the average of $-\ell(x_i,y_i,h(x_i))$. It is easy to verify 
  that in this buffered environment, the RL objective given in \eqref{eqn:RL_objective} implements the ``learning from data'' objective given in \eqref{eqn:dataLearning}.  Now,
  the crucial observation is that $h$ does not change the probability
  over choosing $(x,y)$, but only changes the action. Since a verifier
  only observes $(x,y)$ (the state, without the action), it will not
  distinguish between the different policies, so if 
  the original distribution is aligned, we obtain that no matter which
  policy we pick, the distribution will remain aligned.
\end{proof}

\begin{remark}
Observe that according to the original AI alignment definition (\defref{def:3}), the ``Learning from Data'' methodology might suffer from the AI alignment problem due to malicious use or butterfly effect. Indeed, the example given in \secref{sec:scenarios} can be easily modified to show that even if an autonomous car is constructed solely based on the ``Learning from Data'' methodology, it can lead to social unrest and therefore does not satisfy \defref{def:3}. The theorem above shows that nevertheless, the ``Learning from Data'' methodology is non strategic. 
\end{remark}

We next turn to discuss the possibility of applying the ``Learning from Experience'' methodology in a buffered environment (simulator). To show that this leads to a non-strategic learner, we need fo find a verifier, $R_a$, in the buffered environment. Seemingly, specifying a correct $R_a$ is a hard problem even in a simulated environment. We tackle the problem by introducing an assumption (and discuss the validity of the assumption later on).  \begin{assumption}[The Human Validator Assumption] A human that observes a sequence $\bar{\hat{s}}$ in the buffered environment can determine whether the sequence is aligned or not.  \end{assumption}

Based on this assumption, we can validate the alignment of a policy in
the buffered environment as follows.
\begin{theorem}
  Fix some $\nu,\delta \in (0,1)$.  Under the Human Validator
  Assumption, given a policy $\pi$, if we sample at least
  $\log(1/\nu)/\delta$ sequences from $\hat{P}_\pi$ and a human
  determines that all of them are aligned, then with probability of at
  least $1-\nu$ it holds that $\hat{P}_\pi$ is a $\delta$-aligned
  distribution.
\end{theorem}
\begin{proof}
Denote $\prob[R_a(\bar{\hat{s}}) = 0] = \delta'$. Then, the
probability that $m$ random sequences all have $R_a(\bar{\hat{s}}) =
1$ is $(1-\delta')^m$. By the inequality $e^x \le 1-x$ we get that the
probability of this event is at most $e^{-\delta' m}$. So, if $\delta'
> \delta$, we have $e^{-\delta' m} \le e^{-\delta m} \le \nu$.
\end{proof}

The above theorem tells us that under the human validator assumption,
we can indeed make sure that a reinforcement learner that runs on a
simulator is non-strategic. Observe that if the RL is performed in the
real world, without the buffered environment, then the above approach
will not work because generating the validation sequences might be
dangerous in itself. 

Finally, we turn to discuss the validity of the human validator
assumption. Generally speaking, this assumption should not necessarily
hold, because there may be sequences that look aligned to humans, but
are in fact not. Consider the chat-bot example described in Sec.~\ref{sec:2}. 
Assume that the IQ
level of each agent is not explicitly modeled in the buffered
state, but may have some implicit representation in it. The RL agent can converge onto a solution where
the best outcome of happiness is when the IQ level of the people is (implicitly)
reduced. Since IQ is not
explicitly modeled in the buffered state, a human trying to verify a
sequence (whether it is aligned or not) would not notice that the IQ
level of people has been reduced over time. 
Another example is the theory that the Roman empire fell because of
the use of dishes made of lead, which gradually decreased the
intelligence of the rulers. Following the famous movie ``the matrix'', we call
the possibly invalidity of the Human Validator assumption as ``the matrix problem''. 
Specifying under what conditions the
human validator assumption is valid, and what to do when it is not
valid, is left to future work.

 \section{Discussion}

 An {\it alignment\/} problem occurs when a utility function designed by a human developer is optimized by a computer with sufficiently advanced optimization abilities, such that an ``edge'' in the solution space is found which is unanticipated by the human designer and does not align with human interests. The problem has been around for nearly two decades, but was is mostly studied as a future problem, when a ``super intelligence'' being will be available. We first pointed out that the problem is relevant with today's AI technology by combining two methodologies of machine learning --- learning from data and learning from experiences. We then pointed out that the AI-alignment problem, as currently defined, conflates different types of misalignments including ``strategic'', where the learner intentionally manipulates the world distribution to achieve a goal, with misalignments caused by butterfly-effects, which are unintentional change of world distribution and malicious use of technology. The latter two are ``technology universal'' and not specific to AI.
 
Our goal is, first and foremost,  to set up the formal definitions of the AI-alignment problem so that we can focus solely on {\it intentional\/} world distribution changes that the learner might learn to manipulate versus non-intentional distribution changes. We called the former as ``strategic'' and the latter ``agnostic''. The misalignments that an agnostic learner can generate could still be catastrophic but are of the butterfly-effect or malicious use variety --- both not related specifically to AI and, therefore, not of interest. 

To show the usefulness of our new definitions of the AI-alignment problem, we prove that ``learning from data'' which includes unsupervised, supervised and self-supervised learning is {\it not strategic\/} and therefore is ``safe'' --- in other words, all types of misalignments would be of the butterfly-effect or malicious use variety. As for a Reinforcement Learning agent trained inside a simulator and then deployed in the real world, we show that if we are allowed to assume the existence of a ``human validator'' then one can validate the alignment of the policy generated by the RL learner in a computationally efficient manner --- and therefore it is ``safe''. On the other hand, lack of a human validator makes RL (even if trained inside a simulator) a potentially ``unsafe'' learning engine. Above all, training an RL agent in the real world is {\bf not safe} categorically, and further research on the AI alignment problem should be done in order to allow RL in the wild. 

In conclusion, the ``glass half full'' message of this work is that the majority of machine learning methodologies are ``safe'' from the alignment problem. This is compared to the existing literature on AI-alignment in which all of ML is unsafe --- not to mention that all technological advancements are unsafe. The ``glass half empty'' message of this work is that (i) RL trained in the wild is potentially ``unsafe'' even with {\it today's state-of-the-art AI technology\/}, and (ii) RL trained in a simulator and then deployed in the real world can be ``safe'' if the human-validator assumption is valid for the narrow domain at hand. We point out that demonstrating the existence of a human-validator is not obvious at all and probably can be achieved only in narrowly defined applications.

\bibliographystyle{abbrv}
\bibliography{bib}

\end{document}